\documentclass{article}

\usepackage{amsmath}
\usepackage{amssymb}
\usepackage{amsfonts}

\usepackage{amsthm}
\usepackage{thmtools}
\usepackage{thm-restate}

\usepackage{mathtools}
\usepackage{xspace}

\declaretheorem[name=Theorem]{theorem}
\declaretheorem[name=Lemma,numberlike=theorem]{lemma}

\declaretheorem[name=Definition]{definition}

\newcommand{\R}{\mathbb{R}\xspace}
\newcommand{\X}{\mathcal{X}\xspace}

\allowdisplaybreaks

\usepackage{times}
\usepackage{graphicx}
\usepackage{subcaption}
\usepackage{psfrag}

\usepackage{natbib}

\usepackage{algorithm}
\usepackage{algorithmic}

\usepackage{hyperref}

\usepackage[accepted]{icml2016}

\icmltitlerunning{Single-Solution Hypervolume for Improving Generalization of
NNs}

\begin{document}

\twocolumn[
\icmltitle{Single-Solution Hypervolume Maximization and its use \\
  for Improving Generalization of Neural Networks}

\icmlauthor{Conrado S. Miranda}{contact@conradomiranda.com}
\icmlauthor{Fernando J. Von Zuben}{vonzuben@dca.fee.unicamp.br}
\icmladdress{University of Campinas, Brazil}

\icmlkeywords{hypervolume, multi-objective optimization}

\vskip 0.3in
]

\begin{abstract}
This paper introduces the hypervolume maximization with a single solution as an
alternative to the mean loss minimization.
The relationship between the two problems is proved through bounds on the cost
function when an optimal solution to one of the problems is evaluated on the
other, with a hyperparameter to control the similarity between the two problems.
This same hyperparameter allows higher weight to be placed on samples with
higher loss when computing the hypervolume's gradient, whose normalized version
can range from the mean loss to the max loss.
An experiment on MNIST with a neural network is used to validate the theory
developed, showing that the hypervolume maximization can behave similarly to the
mean loss minimization and can also provide better performance, resulting on a
20\% reduction of the classification error on the test set.
\end{abstract}

\section{Introduction}
\label{sec:introduction}
Many machine learning algorithms, including neural networks, can be divided into
three parts: the model, which is used to describe or approximate the structure
present in the training data set; the loss function, which defines how well an
instance of the model fits the samples; and the optimization method, which
adjusts the model's parameters to improve the error expressed by the loss
function. Obviously these three parts are related, and the generalization
capability of the obtained solution depends on the individual merit of each one
of the three parts, and also on their interplay.

Most of current research in machine learning focuses on creating new
models
\cite{bengio2009learning,Goodfellow-et-al-2016-Book,koller2009probabilistic},
for different applications and data types, and new optimization methods
\cite{bennett2006interplay,dauphin2015equilibrated,duchi2011adaptive,zeiler2012adadelta},
which may allow faster convergence, more robustness, and a better chance to
escape from poor local minima.

On the other hand, many cost functions come from statistical models
\cite{bishop2006pattern}, such as the quadratic error, cross-entropy or
variational bound \cite{kingma2013auto}, although some terms of the cost related
to regularization not necessarily have statistical basis
\cite{soft_svm,miyato2015distributional,rifai2011contractive}. When building the
total cost of a sample set, we frequently sum the costs for each sample plus
regularization terms for the whole dataset. Although this methodology is sound,
it can be problematic in real-world applications involving more complex models.

More specifically, if the learning problem is viewed from a multi-objective
optimization (MOO) perspective as minimization of the cost for each sample
individually, then not every efficient solution may be achieved by a convex
combination of the costs \cite{boyd2004convex} and Pareto-based solutions might
provide better results \cite{Freitas2004}. An alternative to minimizing the
convex combination is to maximize a metric known as the hypervolume
\cite{zitzler2007hypervolume}, which can be used to measure the quality of a set
of samples. As MOO algorithms usually search for many solutions with different
trade-offs of the objectives at the same time \cite{deb2014multi}, which can be
used in an ensemble for instance \cite{Chandra2006}, this ability to evaluate
the whole set of solutions instead of a single one made this metric widely used
in MOO \cite{hypervolume_methods}.

The computation of the hypervolume is NP-complete \cite{beume2009complexity},
making it hard to be used when there are many objectives and candidate
solutions. Nonetheless, in the particular case that a single candidate solution
is being used, it can be computed in linear time with the number of objectives,
which makes its computing time equal to the one associated with a convex
combination.

Under the MOO perspective of having a single objective function per sample, in
this paper we develop a theory linking the maximization of the single-solution
hypervolume to the minimization of the mean loss, in which the average of the
cost over the training samples is considered. We provide theoretical bounds on
the hypervolume value in the neighborhood of an optimal solution to the mean
loss and vice-versa, showing that these bounds can be made arbitrarily small
such that, in the limit, the optimal value for one problem is also optimal for
the other.

Moreover, we analyze the gradient of the hypervolume, showing that it
places more importance to samples with higher cost.
Since gradient optimization is an iterative process, the hypervolume
maximization implies an automatic reweighing of the samples at each iteration.
This reweighing allows the hypervolume gradient to range from the maximum loss'
to the mean loss' gradient by changing a hyperparameter.
It is also different from optimizing a linear combination of the mean and
maximum losses, as it also considers the losses of intermediary samples.

We conjecture that the gradient obtained from the hypervolume guides to improved
models, as it is able to express a compromise between fitting well the average
sample and the worst samples.
The automatic reweighing prevents the learning algorithm from pursuing a quick
reduction in the mean loss if it requires a significant increase in the loss on
already badly represented samples.

We perform an experiment to provide both empirical evidence for the conjecture,
showing that using the hypervolume maximization reduces classification error
when compared to the mean loss minimization, and validation for the theory
developed.

This paper is organized as follows. Section~\ref{sec:moo} provides an overview
of multi-objective optimization, properly charactering the maximization of the
hypervolume as a performance indicator for the learning system.
Section~\ref{sec:theory} presents the theoretical developments of the paper and
Section~\ref{sec:experiment} describes the experiment performed to validate the
theory and provides evidence for conjectures developed in the paper. Finally,
Section~\ref{sec:conclusion} outlines concluding remarks and future research
directions.
\section{Multi-objective optimization}
\label{sec:moo}
Multi-objective optimization (MOO) is a generalization of the traditional
single-objective optimization, where the problem is composed of multiple
objective functions $f_i \colon \X \to \mathcal \R,i \in [N]$, where $[N] =
\{1,2,\ldots,N\}$ \cite{deb2014multi}. Using the standard notation for MOO, the
problem can be described by:
\begin{equation}
  \label{eq:moo}
  \min_{x \in \X} f_1(x),
  \ldots,
  \min_{x \in \X} f_N(x),
\end{equation}
where $\X$ is the decision space and includes all constraints of the
optimization.

If some of the objectives have the same minima, then the redundant objectives
can be ignored during optimization. However, if their minima are different, for
example $f_1(x) = x^2$ and $f_2(x) = (x-1)^2$, then there is not a single
optimal point, but a set of different trade-offs between the objectives. A
solution that establishes an optimal trade-off, so that it is impossible to
reduce one of the objectives without increasing another, is said to be
efficient. The set of efficient solutions is called the Pareto set and its
counterpart in the objective space is called the Pareto frontier.

\subsection{Linear combination}
\label{sec:moo:combination}
A common approach in optimization used to deal with multi-objective problems is
to combine the objectives linearly \cite{boyd2004convex,deb2014multi}, so that
the problem becomes
\begin{equation}
  \label{eq:linear_combination}
  \min_{x \in \X} \sum_{i=1}^N w_i f_i(x),
\end{equation}
where the weight $w_i \in \R^+$ represents the relative importance given to
objective $i \in [N]$.

This approach is frequently found in learning with regularization
\cite{bishop2006pattern}, where one objective is to decrease the loss on the
training set and another is to decrease the model complexity, and the multiple
objectives are combined with weights for the regularization terms to balance the
trade-off. Examples of this technique include soft-margin support vector
machines \cite{soft_svm}, semi-supervised models \cite{rasmus2015semi}, and
adversarial examples \cite{miyato2015distributional}, among many others.

Although the optimal solution of the linearly combined problem is guaranteed to
be efficient, it is only possible to achieve any efficient solution when the
objectives are convex \cite{boyd2004convex}. This means that some desired
solutions may not be achievable by performing a linear combination of the
objectives and Pareto-based approaches should be used \cite{Freitas2004}, which
led to the creation of the hypervolume indicator in MOO.

\subsection{Hypervolume indicator}
\label{sec:moo:hypervolume}
Since the linear combination of objectives is not going to work properly on
non-convex objectives, it is desirable to investigate other forms of
transforming the multi-objective problem into a single-objective one, which
allows the standard optimization tools to be used.

One common approach in the multi-objective literature is to resort to the
hypervolume indicator \cite{zitzler2007hypervolume}, given by
\begin{equation}
  \mathcal H (z, X) = \int_{\mathcal Y}
  1[\exists x \in X \colon f(x) \prec y \prec z] \mathrm{d}y,
\end{equation}
where $z \in \R^N$ is the reference point, $X \subseteq \X$, $f(\cdot)$ is
the vector obtained by stacking the objectives, $\prec$ is the dominance
operator \cite{deb2014multi}, which is similar to the $<$ comparator and can be
defined as $x \prec y \Leftrightarrow (x_1 < y_1) \wedge \ldots \wedge (x_N <
y_N)$, and $1[\cdot]$ is the indicator operator. The problem then becomes
maximizing the hypervolume over the domain, and this optimization is able to
achieve a larger number of efficient points, without requiring convexity of the
Pareto frontier \cite{auger2009theory}.

Although the hypervolume is frequently used to analyze a set of candidate
solutions \cite{zitzler2003performance} and led to state-of-the-art algorithms
for MOO \cite{hypervolume_methods}, it can be expensive to compute as it is
NP-complete \cite{beume2009complexity}. However, for a single solution, that is,
if $|X|=1$, it can be computed in linear time and its logarithm can be written
as:
\begin{equation}
  \label{eq:log_hypervolume}
  \log \mathcal H(z, \{x\}) = \sum_{i=1}^N \log(z_i - f_i(x)),
\end{equation}
given that $f_i(x) < z_i$.

Among the many properties of the hypervolume, two must be highlighted in this
paper. First, the hypervolume is monotonic in the objectives, which means that
any reduction of any objective causes the hypervolume to increase, which in turn
is aligned with the loss minimization. The maximum of the single-solution
hypervolume is a point in the Pareto set, which means that the solution is
efficient.

The second property is that, like the linear combination, it also maintains some
shape information from the objectives. If the objectives are convex, then their
linear combination is convex and the hypervolume is concave, since $-f_i(x)$ is
concave and the logarithm of a concave function is concave.

\subsection{Loss minimization}
A common objective in machine learning is the minimization of some loss function
$l \colon D \times \Theta \to \R$ over a given data set $S = \{s_1,\ldots,s_N\}
\subset D$. Note that this notation includes both supervised and unsupervised
learning, as the space $D$ can include both the samples and their targets. For
simplicity, let $l_i(\theta) \coloneqq l(s_i, \theta)$, so that the specific
data set does not have to be considered.

Defining $f_i \coloneqq l_i$ and $\X \coloneqq \Theta$, the loss minimization
can be written as Eq.~\eqref{eq:moo}. Just like in other areas of optimization,
the usual approach to solve these problems in machine learning is the use of a
linear combination of the objectives, as discussed in
Sec.~\ref{sec:moo:combination}. However, as also discussed in
Sec.~\ref{sec:moo:combination}, this approach limits the number of solutions
that can be obtained if the losses are not convex, which motivates the use of
the hypervolume indicator.

Since the objectives differ only
in the samples used for the loss function and considering that all samples have
equal importance\footnote{This is the same motivation for using the uniform mean
loss. If prior importance is available, it can be used to define the value of
$z$, like it would be used to define $w$ in the weighted mean loss.}, the Nadir
point $z$ can have the same value for all objectives so that the solution found
is balanced in the losses. This value is given by the parameter $\mu$, so that
$z_i = \mu, \forall i \in [N]$. Then the problem becomes maximizing $\log
\mathcal H(\mu 1_N, \{\theta\})$ in relation to $\theta$, where $1_N$ is a
vector of ones with size $N$ and $\log \mathcal H(\cdot)$ is defined in
Eq.~\eqref{eq:log_hypervolume}.
\section{Theory of the single-solution hypervolume}
\label{sec:theory}

In this section, we develop the theory linking the single-solution hypervolume
maximization to the minimization of the mean loss, which is
a common optimization objective.
First, we define the requirements that a loss function must satisfy and describe
the two optimization problems in Sec.~\ref{sec:theory:definitions}.
Then, given an optimal solution to one problem, we will show in
Sec.~\ref{sec:theory:connection} bounds on the loss of optimality of the other
problem in the neighborhood of the given solution and will show that these
bounds can be made arbitrarily small by changing the reference point.
Finally in Sec.~\ref{sec:theory:gradient}, we will show how to transform the
gradient of the hypervolume to a format similar to a convex combination of the
gradients of each loss, which will be used in the experiments of
Sec.~\ref{sec:experiment} to show the advantage of using the hypervolume
maximization.

\subsection{Definitions}
\label{sec:theory:definitions}
In order to elaborate the theory, we must define some terms that will be used on
the results.

\begin{definition}[Loss Function]
  Let $\Theta$ be an open subset of $\mathbb R^n$. Let $l \colon \Theta \to
  \mathbb R$ be a continuously differentiable function. Then $l$ is a loss
  function if the following conditions hold:
  \begin{itemize}
    \item The loss is bounded everywhere, that is, $|l(\theta)| < \infty$ for
      all $\theta \in \Theta$;
    \item The gradient is bounded everywhere, that is, $\|\nabla l(\theta)\| <
      \infty$ for all $\theta \in \Theta$.
  \end{itemize}
\end{definition}

The openness of $\Theta$ simplifies the theory as we do not have to worry about
optima on the border, which are harder to deal with during proofs. However, the
theory can be adjusted so that $\Theta$ can be closed and points on the border
are allowed to have infinite loss.

\begin{definition}[Mean loss problem]
  Let $\Theta$ be an open subset of $\mathbb R^n$.
  Let $L = \{l_1, \ldots, l_N\}$ be a set of loss functions defined over
  $\Theta$.
  Then the mean loss $J_m \colon \Theta \to \mathbb R$ and its associated
  minimization problem are defined as
  \begin{equation}
    \label{eq:cost_mean}
    \min_{\theta \in \Theta} J_m(\theta), \quad
    J_m(\theta) = \frac{1}{N} \sum_{i=1}^N l_i(\theta).
  \end{equation}
\end{definition}

\begin{definition}[Hypervolume problem]
  Let $\Theta$ be an open subset of $\mathbb R^n$.
  Let $L = \{l_1, \ldots, l_N\}$ be a set of loss functions defined over
  $\Theta$.
  Let $\mu$ be given such that there exists some $\theta \in \Theta$ that
  satisfies $\mu > l(\theta)$ for all $l \in L$. Let $\Theta' = \{\theta \mid
  \theta \in \Theta, \mu > l(\theta), \forall l \in L\}$. Then the hypervolume
  $H \colon \mathbb R \times \Theta' \to \mathbb R$
  and its associated maximization problem are defined as
  \begin{equation}
    \label{eq:cost_H}
    \max_{\theta \in \Theta'} H(\mu, \theta), \quad
    H(\mu,\theta) = \sum_{i=1}^N \log(\mu - l_i(\theta)).
  \end{equation}
\end{definition}

Note that the hypervolume defined here is a simplification of the function
defined in Eq.~\eqref{eq:log_hypervolume}, so that $H(\mu, \theta) \coloneqq
\log \mathcal H(\mu 1_N, \{\theta\})$.

\subsection{Connection between $J_m(\theta)$ and $H(\mu, \theta)$}
\label{sec:theory:connection}
Using the definitions in the last section, we can now state bounds when applying
the optimal solution of a problem to the other.
The proofs are not present in this section in order to avoid cluttering, but
are provided in Appendix~\ref{sec:proofs}.

\begin{restatable}{theorem}{MeanToHMain}
  \label{theorem:mean_to_h_main}
  Let $\Theta$ be an open subset of $\mathbb R^n$.
  Let $L = \{l_1, \ldots, l_N\}$ be a set of loss functions defined over
  $\Theta$.
  Let $\theta^* \in \Theta$ be a local minimum of $J_m(\theta)$ and let
  $\epsilon > 0$ such that $\theta^* + \delta \in \Theta$ for all $\|\delta\|
  \le \epsilon$.
  Let $C_1$, $C_2$ and $C_3$ be such that $C_1 \le l_i(\theta^* + \delta) \le
  C_2$ and $\|\nabla l_i(\theta^* + \delta)\| \le C_3$ for all $i \in [N]$ and
  $\|\delta\| \le \epsilon$.
  Let $\nu > 0$.
  Then there is some $\epsilon' \in (0, \epsilon]$ such that
  \begin{equation}
    \label{eq:mean_to_h_main:bound}
    H(\mu, \theta^* + \delta) \le
    H(\mu, \theta^*) + \frac{\nu C_3 \epsilon' N}{\mu - C_2},
  \end{equation}
  for all $\|\delta\| \le \epsilon'$ and $\mu > \gamma$, where
  \begin{equation}
    \gamma = \max\left\{
      C_2,
      \frac{(1+\nu) C_2 - C_1}{\nu},
      \frac{C_2 - (1-\nu) C_1}{\nu}
    \right\}.
  \end{equation}
\end{restatable}
\begin{restatable}{theorem}{HToMeanMain}
  \label{theorem:h_to_mean_main}
  Let $\Theta$ be an open subset of $\mathbb R^n$.
  Let $L = \{l_1, \ldots, l_N\}$ be a set of loss functions defined over
  $\Theta$.
  Let $\theta^* \in \Theta$ be a local maximum of $H(\mu, \theta)$ for some
  $\mu$ and let $\epsilon > 0$ such that $\theta^* + \delta \in \Theta$ for all
  $\|\delta\| \le \epsilon$.
  Let $C_1$, $C_2$ and $C_3$ be such that $C_1 \le l_i(\theta^* + \delta) \le
  C_2$ and $\|\nabla l_i(\theta^* + \delta)\| \le C_3$ for all $i \in [N]$ and
  $\|\delta\| \le \epsilon$.
  Then there is some $\epsilon' \in (0, \epsilon]$ such that
  \begin{equation}
    \label{eq:h_to_mean_main:bound}
    J_m(\theta^* + \delta) \ge J_m(\theta^*) - \nu C_3 \epsilon'
  \end{equation}
  for all $\|\delta\| \le \epsilon'$, where
  \begin{equation}
    \nu = \max\left\{
      \frac{\mu - C_1}{\mu - C_2} - 1,
      1 - \frac{\mu - C_2}{\mu - C_1}
    \right\}.
  \end{equation}
\end{restatable}

Note that $\nu$ in Eqs.~\eqref{eq:mean_to_h_main:bound} and
\eqref{eq:h_to_mean_main:bound} is multiplying the regular bounds due to the
continuous differentiability of the functions. If $\nu \ge 1$, then the
knowledge that a given $\theta^*$ is optimal in the other problem does not
provide any additional information. However, $\nu$ can be made arbitrarily small
by making $\mu$ large, so increasing $\mu$ allows more information to be shared
among the problems as their loss surfaces become closer.

One practical application of these theorems is that, given a value $\nu \in
(0,1)$ and a region $\Omega \subseteq \Theta$, we can check whether the
reference point $\mu$ is large enough for all $\theta \in \Omega$. If it is,
then optimizing $H(\mu,\theta)$ over $\Omega$ is similar to optimizing a bound
on $J_m(\theta)$ around the optimal solution and vice-versa, with the difference
between the bound and the real value vanishing as $\nu$ gets smaller and $\mu$
gets larger.

\subsection{Gradient of $H(\mu, \theta)$ in the limit}
\label{sec:theory:gradient}
The gradient of the hypervolume, as defined in Eq.~\eqref{eq:cost_H}, is given
by:
\begin{equation}
  \nabla_\theta H(\mu, \theta) = - \sum_{i=1}^N \frac{1}{\mu - l_i(\theta)}
  \nabla l_i(\theta).
\end{equation}
Note that using the hypervolume automatically places more importance on samples
with higher loss during optimization.

We conjecture that this automatic control of relevance is beneficial for
learning a function with better generalization, as the model will be forced to
focus more of its capacity on samples that it is not able to represent well.
Moreover, the hyperparameter $\mu$ provides some control over how much
difference of importance can be placed on the samples, as will be shown below.
This is similar to soft-margin support vector machines \cite{soft_svm}, where we
can change the regularization hyperparameter to control how much the margin can
be reduced in order to better fit misclassified samples. We provide empirical
evidence for this conjecture in Sec.~\ref{sec:experiment}.

For a given $\theta$, this gradient can change a lot by changing $\mu$, which
should be avoided during the optimization in real scenarios.
In order to stabilize the gradient and make it similar to the gradient of a
convex combination of the objective's gradients, we can use
\begin{equation}
  \label{eq:theory:gradient}
  \frac{\nabla_\theta H(\mu, \theta)}
  {\sum_{i=1}^N \frac{1}{\mu - l_i(\theta)}} =
  - \sum_{i=1}^N w_i \nabla l_i(\theta), \quad
  w_i = \frac{\frac{1}{\mu - l_i(\theta)}}{\sum_{j=1}^N \frac{1}{\mu -
  l_j(\theta)}},
\end{equation}
so that $w_i \ge 0$ for all $i \in [N]$ and $\sum_{i=1}^N w_i = 1$.

When the reference $\mu$ either becomes large or close to its allowed lower
bound, this normalized gradient presents interesting behaviors.

\begin{lemma}
  Let $\Theta$ be an open subset of $\mathbb R^n$.
  Let $L = \{l_1, \ldots, l_N\}$ be a set of loss functions defined over
  $\Theta$.
  Let $\theta \in \Theta$. Then
  \begin{equation}
    \lim_{\mu \to \infty}
    \frac{\nabla_\theta H(\mu, \theta)}
    {\sum_{i=1}^N \frac{1}{\mu - l_i(\theta)}}
    = -\nabla J_m(\theta).
  \end{equation}
\end{lemma}
\begin{proof}
  Let $\alpha_i \coloneqq \left(\frac{1}{\mu -
  l_i(\theta)}\right)/\left(\sum_{j=1}^N \frac{1}{\mu - l_j(\theta)}\right)$.
  Then $\lim_{\mu \to \infty} \alpha_i = \frac{1}{N}$, which gives the
  hypervolume limit as:
  \begin{subequations}
  \begin{align}
    &\textstyle
    \lim_{\mu \to \infty}
    \frac{\nabla_\theta H(\mu, \theta)}
    {\sum_{i=1}^N \frac{1}{\mu - l_i(\theta)}}
    \\
    &=\textstyle
    \lim_{\mu \to \infty}
    -\sum_{i=1}^N \alpha_i \nabla l_i(\theta)
    = - \nabla J_m(\theta).
  \end{align}
  \end{subequations}
\end{proof}
\begin{lemma}
  Let $\Theta$ be an open subset of $\mathbb R^n$.
  Let $L = \{l_1, \ldots, l_N\}$ be a set of loss functions defined over
  $\Theta$.
  Let $\theta \in \Theta$.
  Let $\Delta(\theta) \coloneqq \max_{i \in [N]} l_i(\theta)$ and
  $S = \{i \mid i \in [N], \Delta(\theta) = l_i(\theta)\}$. Then
  \begin{equation}
    \lim_{\mu \to \Delta(\theta)^+}
    \frac{\nabla_\theta H(\mu, \theta)}
    {\sum_{i=1}^N \frac{1}{\mu - l_i(\theta)}}
    = -\frac{1}{|S|} \sum_{i \in S} \nabla l_i(\theta)
  \end{equation}
\end{lemma}
\begin{proof}
  Let $\alpha_i \coloneqq \left(\frac{1}{\mu -
  l_i(\theta)}\right)/\left(\sum_{j=1}^N \frac{1}{\mu - l_j(\theta)}\right)$.
  Then $\lim_{\mu \to \Delta(\theta)^+} \alpha_i = 1[i \in S]/|S|$, which gives
  the hypervolume limit as:
  \begin{subequations}
  \begin{align}
    &\textstyle
    \lim_{\mu \to \Delta(\theta)^+}
    \frac{\nabla_\theta H(\mu, \theta)}
    {\sum_{i=1}^N \frac{1}{\mu - l_i(\theta)}}
    \\
    &=\textstyle
    \lim_{\mu \to \Delta(\theta)^+}
    -\sum_{i=1}^N \alpha_i \nabla l_i(\theta)
    = -\frac{1}{|S|} \sum_{i \in S} \nabla l_i(\theta).
  \end{align}
  \end{subequations}
\end{proof}

As shown in Sec.~\ref{sec:theory:connection}, the mean loss and hypervolume
problems become closer as $\mu$ increases. In the limit, the normalized gradient
for the hypervolume becomes equal to the gradient of the mean loss. On the other
hand, when $\mu$ is close to its lower bound $\Delta(\theta)$, it becomes the
mean gradient of all the loss functions with maximum value. In particular, if
$|S|=1$, that is, only one loss has maximal value at some $\theta$, then the
normalized gradient for the hypervolume becomes equal to the gradient of the
maximum loss.
\begin{figure*}[t]
  \centering
  \begin{subfigure}{0.40\linewidth}
    \psfrag{it}[t][c]{\scriptsize Iteration}
    \psfrag{error}[c][c]{\scriptsize Classification error}
    \psfrag{inf}[l][l]{\tiny $\infty$}
    \psfrag{0}[l][l]{\scriptsize 0}
    \psfrag{--1}[l][l]{\scriptsize -1}
    \psfrag{--2}[l][l]{\scriptsize -2}
    \psfrag{--3}[l][l]{\scriptsize -3}
    \psfrag{--4}[l][l]{\scriptsize -4}
    \includegraphics[width=\linewidth]{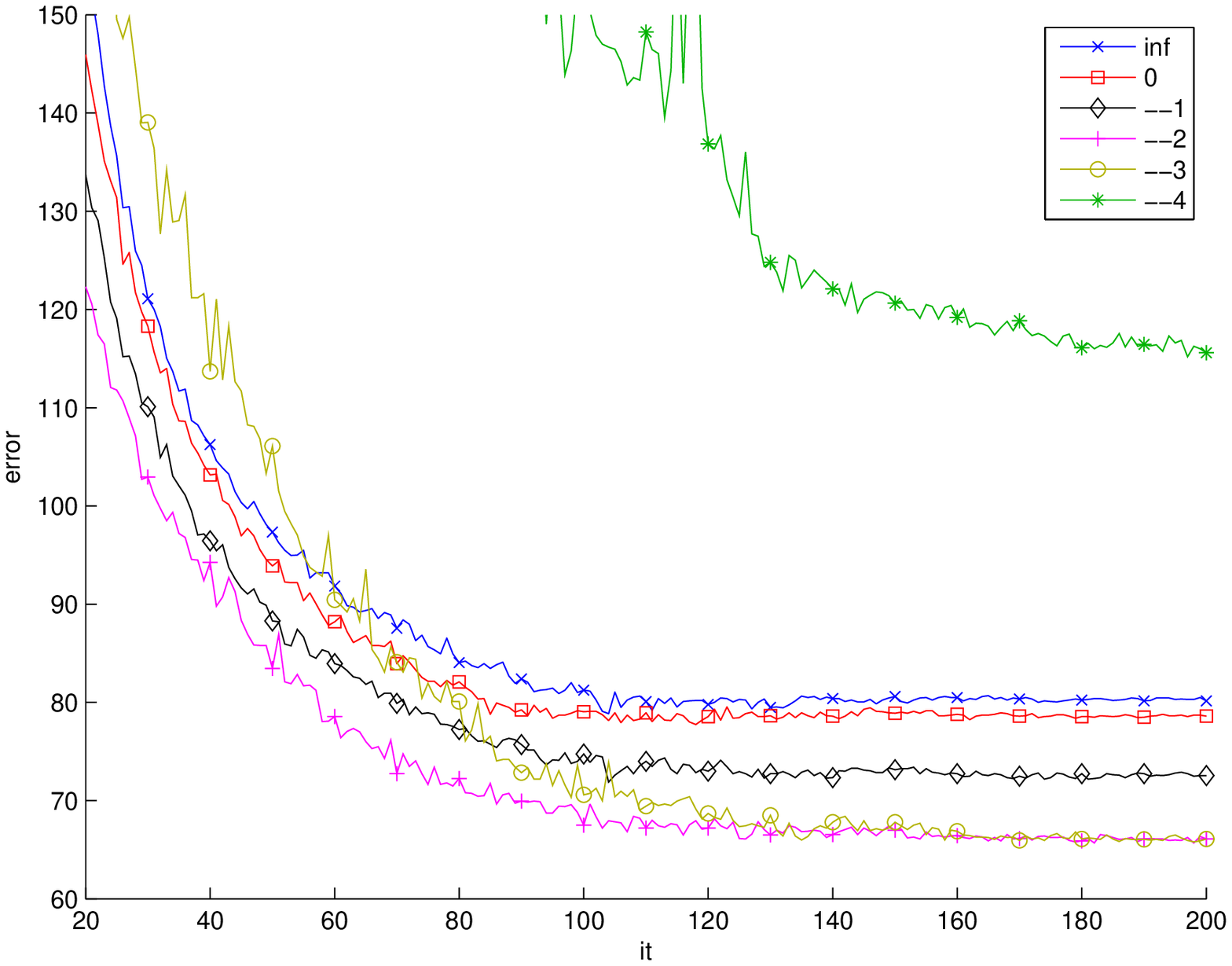}
    \caption{$\Xi = \{\xi_0\}$}
    \label{fig:results:00}
  \end{subfigure}
  \begin{subfigure}{0.40\linewidth}
    \psfrag{it}[t][c]{\scriptsize Iteration}
    \psfrag{error}[c][c]{\scriptsize Classification error}
    \psfrag{inf}[l][l]{\tiny $\infty$}
    \psfrag{0}[l][l]{\scriptsize 0}
    \psfrag{--1}[l][l]{\scriptsize -1}
    \psfrag{--2}[l][l]{\scriptsize -2}
    \psfrag{--3}[l][l]{\scriptsize -3}
    \psfrag{--4}[l][l]{\scriptsize -4}
    \includegraphics[width=\linewidth]{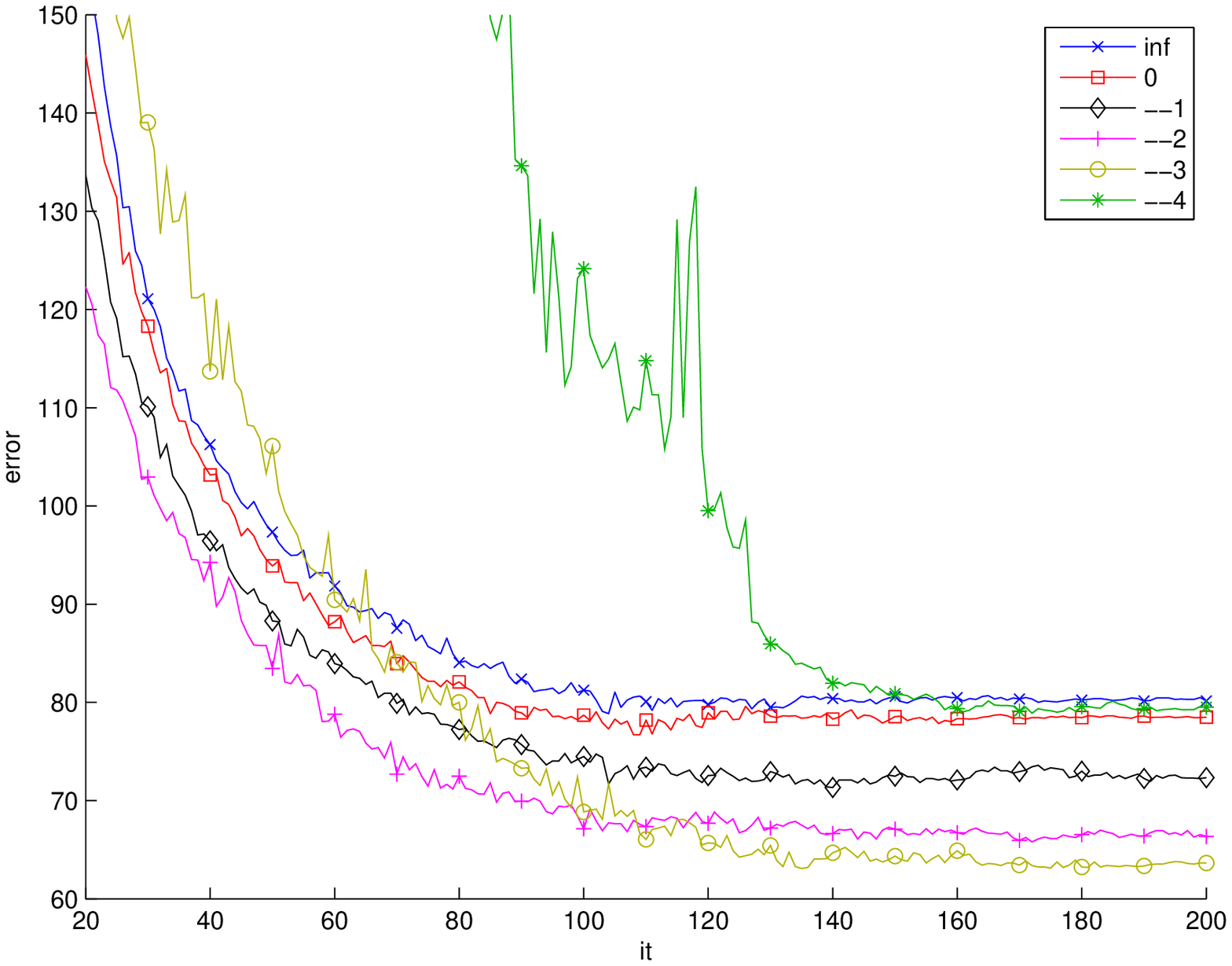}
    \caption{$\Xi = \{\xi_0, \infty\}$}
    \label{fig:results:01}
  \end{subfigure}
  \\
  \begin{subfigure}{0.40\linewidth}
    \psfrag{it}[t][c]{\scriptsize Iteration}
    \psfrag{error}[c][c]{\scriptsize Classification error}
    \psfrag{inf}[l][l]{\tiny $\infty$}
    \psfrag{0}[l][l]{\scriptsize 0}
    \psfrag{--1}[l][l]{\scriptsize -1}
    \psfrag{--2}[l][l]{\scriptsize -2}
    \psfrag{--3}[l][l]{\scriptsize -3}
    \psfrag{--4}[l][l]{\scriptsize -4}
    \includegraphics[width=\linewidth]{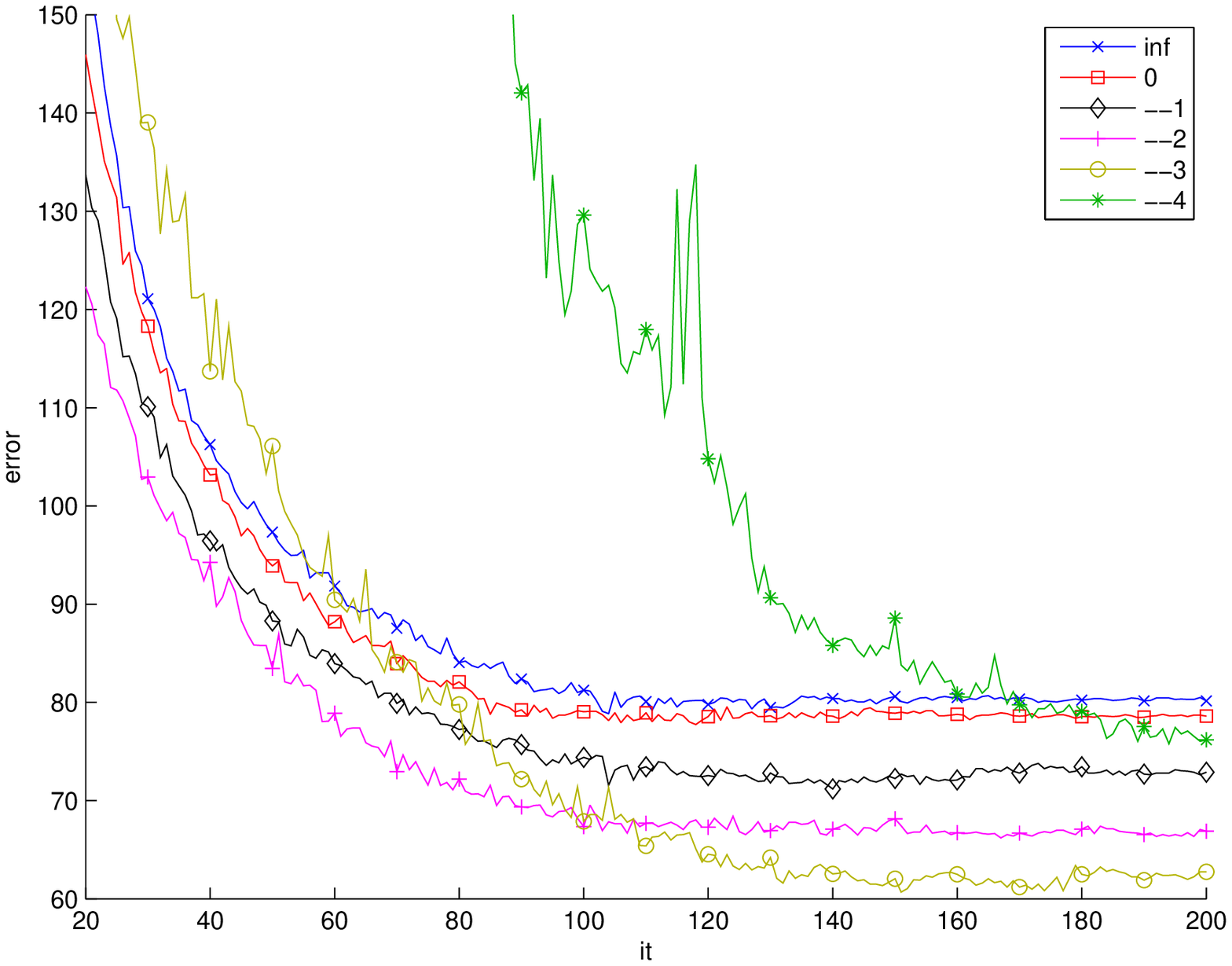}
    \caption{$\Xi = \{\xi_0,\ldots,0\}$}
    \label{fig:results:10}
  \end{subfigure}
  \begin{subfigure}{0.40\linewidth}
    \psfrag{it}[t][c]{\scriptsize Iteration}
    \psfrag{error}[c][c]{\scriptsize Classification error}
    \psfrag{inf}[l][l]{\tiny $\infty$}
    \psfrag{0}[l][l]{\scriptsize 0}
    \psfrag{--1}[l][l]{\scriptsize -1}
    \psfrag{--2}[l][l]{\scriptsize -2}
    \psfrag{--3}[l][l]{\scriptsize -3}
    \psfrag{--4}[l][l]{\scriptsize -4}
    \includegraphics[width=\linewidth]{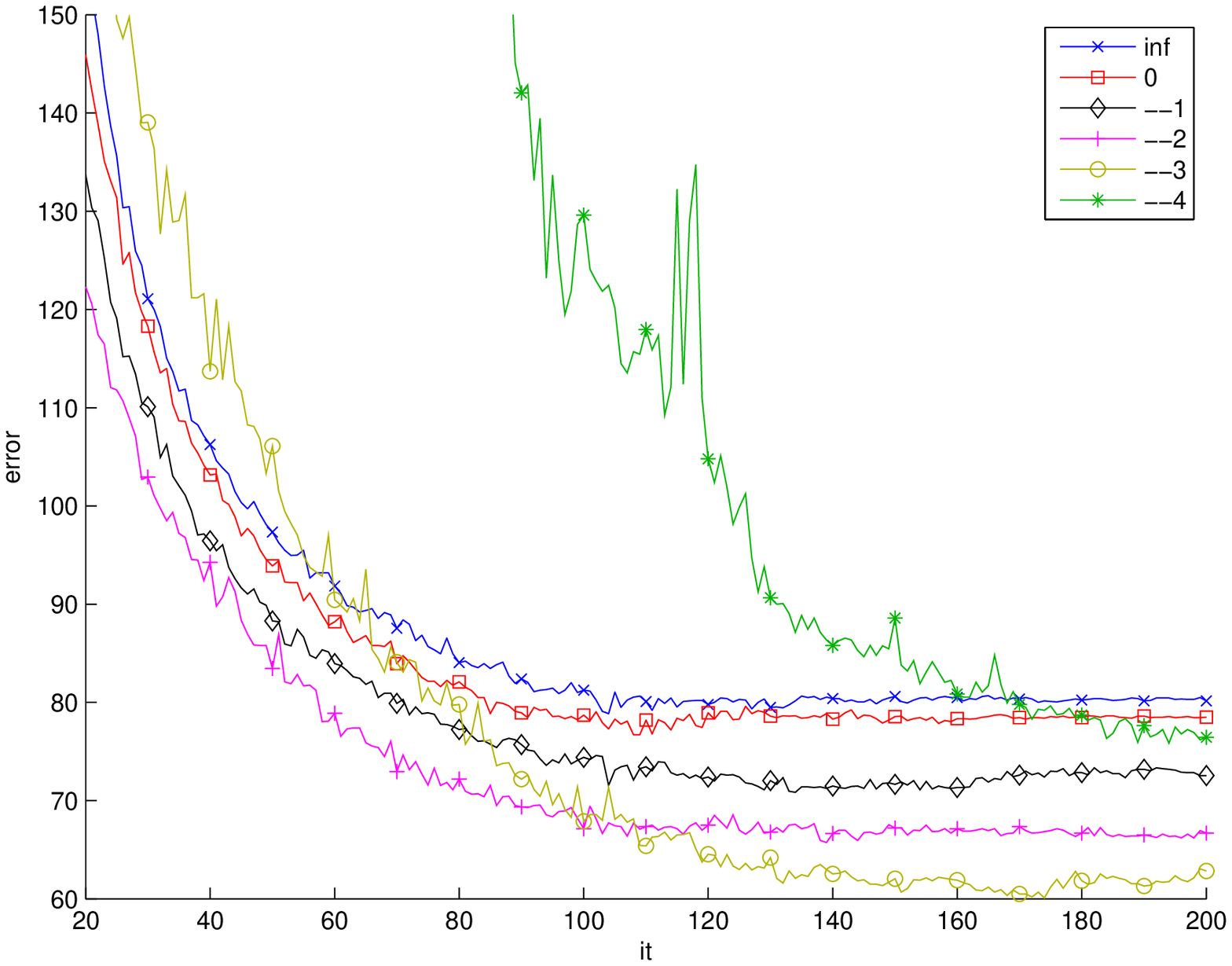}
    \caption{$\Xi = \{\xi_0,\ldots,0,\infty\}$}
    \label{fig:results:11}
  \end{subfigure}
  \caption{Mean number of misclassified samples in the test set over 20 runs for
   different initial values $\xi_0$ and training strategies, with $\xi =
   \infty$ representing the mean loss and $\Xi$ representing the schedule of
   values of $\xi$ used when no improvement is observed.}
  \label{fig:results}
\end{figure*}
\section{Experimental validation}
\label{sec:experiment}

To validate the use of the single-solution hypervolume instead of the mean loss
for training neural networks, we used a LeNet-like network on the MNIST dataset
\cite{lecun1998gradient}. This network is composed of three layers, all with
ReLU activation, where the first two layers are convolutions with 20 and 50
filters, respectively, of size 5x5, both followed by max-pooling of size 2x2,
while the last layer is fully connected and composed of 500 hidden units with
dropout probability of 0.5. The learning was performed by gradient descent with
base learning rate of 0.1 and momentum of 0.9, which were selected using the
validation set to provide the best performance for the mean loss optimization,
and minibatches of 500 samples. After 20 iterations without improvement on the
validation error, the learning rate is reduced by a factor of 0.1 until it
reaches the value of 0.001, after which it is kept constant until 200 iterations
occurred.

For the hypervolume, instead of fixing a single value for $\mu$, which would
require it to be large as the neural network has high loss when initialized, we
allowed $\mu$ to change as
\begin{equation}
  \mu = (1 + 10^\xi) \max_i l_i(\theta)
\end{equation}
so that it can follow the improvement on the loss functions, where $i$ are the
samples in the mini-batch and the parameters' gradients are not backpropagated
through $\mu$. Any value $\xi \in \mathbb R \cup \{\infty\}$ provides a valid
reference point and larger
values make the problem closer to using the mean loss. We tested $\xi \in \Xi =
\{-4, -3, -2, -1, 0, \infty\}$, where $\xi = \infty$ represents the mean loss,
and allowed for scheduling of $\xi$. In this case, before decreasing the
learning rate when the learning stalled, $\xi$ is incremented to the next value
in $\Xi$. We have also considered the possibility of $\infty \notin \Xi$, to
evaluate the effect of not using the mean together with the schedule.

Figure~\ref{fig:results} shows the results for each scenario considered. First,
note that using $\xi_0 = 0$ provided results close to the mean loss throughout
the iterations on all scenarios, which empirically validates the theory that
large values of $\mu$ makes maximization of the hypervolume similar to
minimization of the mean loss and provides evidence that $\mu$ does not have to
be so large in comparison to the loss functions for this to happen. Moreover,
Figs.~\ref{fig:results:10} and \ref{fig:results:11} are similar for all values
of $\xi_0$, which provides further evidence that $\xi=0$ is large enough to
approximate the mean loss well, as including $\xi = \infty$ in the schedule or
not does not change the performance.

On the other hand, $\xi_0 = -4$ was not able to provide good classification by
itself, requiring the use of other values of $\xi$ to achieve an error rate
similar to the mean loss. Although it is able to get better results with the
help of schedule, as shown in Figs.~\ref{fig:results:10} and
\ref{fig:results:11}, this probably is due to the other values of $\xi$
themselves instead of $\xi_0 = -4$ providing direct benefit, as it achieved an
error similar to the mean loss when no schedule except for $\xi = \infty$ was
used, as shown in Fig.~\ref{fig:results:01}. This indicates that too much
pressure on the samples with high loss is not beneficial, which is explained by
the optimization becoming closed to minimizing the maximum loss, as discussed in
Sec.~\ref{sec:theory:gradient}, thus ignoring most of the samples.

When optimizing the hypervolume starting from $\xi_0 \in \{-1, -2, -3\}$, all
scenarios showed improvements on the classification error, with all the
differences after convergence between optimizing the hypervolume and the mean
loss being statistically significant. Moreover, better results were obtained
when the schedule started from a smaller value of $\xi_0$. This provides
evidence to the conjecture in Sec.~\ref{sec:theory:gradient} that placing higher
pressure on samples with high loss, which is represented by higher values of
$w_i$ in Eq.~\eqref{eq:theory:gradient}, is beneficial and might help the
network to achieve higher generalization, but also warns that too much pressure
can be prejudicial as the results for $\xi_0 = -4$ show.

Furthermore, Fig.~\ref{fig:results} indicates that, even if this pressure is
kept throughout the training, it might improve the results compared to using
only the mean loss, but that reducing the pressure as the training progresses
improves the results. We suspect that reducing the pressure allows rare
samples that cannot be well learned by the model to be less relevant in favor of
more common samples, which improves the generalization overall, and that the
initial pressure allowed the model to learn better representations for the data,
as it was forced to consider more the bad samples. The presence of these rare
and bad samples also explain why $\xi_0 = -4$ provided bad results, as the
learning algorithm focused mainly on samples that cannot be appropriately learnt
by the model instead of focusing on the more representative ones.

\begin{table}[t]
  \centering
  \caption{Mean number of misclassified samples in the test set over 20 runs.
  The differences between $\xi_0 = \infty$ and $\xi_0 = -3$ are statistically
  significant ($p \ll 0.001$).}
  \label{tab:test_errors}
  \begin{tabular}{|c|c|c|c|c|}
    \hline
    $\xi_0$ & Schedule & Mean & Errors & Reduction
    \\
    \hline
    $\infty$ & & & 80.8 &
    \\
    $-3$ & X & X & 67.5 & 16.5\%
    \\
    $-3$ & X & \checkmark & 64.2 & 20.5\%
    \\
    $-3$ & \checkmark & X & 62.9 & 22.2\%
    \\
    $-3$ & \checkmark & \checkmark & 63.4 & 21.6\%
    \\
    \hline
  \end{tabular}
\end{table}

Table~\ref{tab:test_errors} provides the errors for the mean loss, represented
by $\xi_0 = \infty$, and for hypervolume with $\xi_0 = -3$, which presented the
best improvements. We used the classification error on the validation set to
select the parameters used for computing the classification error on the test
set. If not used alone, with either scheduling or mean or both, maximizing the
hypervolume leads to a reduction of at least $20\%$ in the classification error
without changing the convergence time significantly, as observed in
Fig.~\ref{fig:results}, which motivates its use in real problems.
\section{Conclusion}
\label{sec:conclusion}
In this paper, we introduced the idea of using the hypervolume with a single
solution as an optimization objective and presented a theory for its use. We
showed how an optimal solution for the hypervolume relates to the mean loss
problem, where we try to minimize the average of the losses for each sample, and
vice-versa, providing bounds on the neighborhood of the optimal point. We also
showed how the gradient of the hypervolume behaves when changing the reference
point and how to stabilize it for practical applications.

This analysis raised the conjecture that using the hypervolume in machine
learning might result in better models, as the hypervolume's gradient is
composed of an automatically weighted average of the gradient for each sample
with higher weights representing higher losses.
This weighting makes the learning algorithm focus more on samples that are not
well represented by the current set of parameters even if it means a slower
reduction of the mean loss.
Hence, it forces the learning algorithm to search for regions where all samples
can be well represented, avoiding early commitment to less promising regions.

Both the theory and the conjecture were validated in an experiment with MNIST,
where using the hypervolume maximization led to a reduction of 20\% in the
classification errors in comparison to the mean loss minimization.

Future research should focus on studying more theoretical and empirical
properties of the single-solution hypervolume maximization, to provide a solid
explanation for its improvement over the mean loss and in which scenarios this
could be expected. The robustness of the method should also be investigated, as
too much noise or the presence of outliers might cause large losses, which opens
the possibility of inducing the learner to place high importance on these losses
in detriment of more common cases.

\appendix
\section{Proofs}
\label{sec:proofs}
\subsection{Proof of Theorem~\ref{theorem:mean_to_h_main}}
\begin{lemma}
  \label{lemma:mean_to_h_helper}
  Let $\Theta$ be an open subset of $\mathbb R^n$.
  Let $L = \{l_1, \ldots, l_N\}$ be a set of loss functions defined over
  $\Theta$.
  Let $\theta^* \in \Theta$ be a local minimum of $J_m(\theta)$. Then there is
  some $\epsilon > 0$ such that, for all $\Delta$ with $\|\Delta\| \le
  \epsilon$, we have $\theta^* + \Delta \in \Theta$ and $\sum_{i=1}^N \nabla
  l_i(\theta^* + \Delta) \cdot \Delta \ge 0$.
\end{lemma}
\begin{proof}
  Since $\Theta$ is open, there is some $\sigma > 0$ such that $\theta^* +
  \delta \in \Theta$ for all $\|\delta\| \le \sigma$.
  Since $\theta^*$ is a local minimum of $J_m(\theta)$, there is some $\epsilon'
  \in (0, \sigma]$ such that $J_m(\theta^* + \delta) \ge J_m(\theta^*)$ for all
  $\|\delta\| \le \epsilon'$. Given some $\delta \ne 0$, from the mean value
  theorem we have that
  \begin{equation}
    \label{eq:J_MVT}
    J_m(\theta^* + \delta) - J_m(\theta^*)
    = \nabla J_m(\theta^* + \Delta) \cdot \Delta / c(\delta)
  \end{equation}
  for some $c(\delta) \in (0,1)$, where $\Delta = c(\delta)\delta$.
  From the optimality, we have
  \begin{equation}
    0 \le J_m(\theta^* + \delta) - J_m(\theta^*)
    = \frac{1}{N} \sum_{i=1}^N \nabla l_i(\theta^* + \Delta) \cdot \Delta /
    c(\delta).
  \end{equation}
  If $\delta = 0$, then the equality is trivially satisfied.
  Let $\kappa = \min_{\|\delta\| \le \epsilon'} c(\delta)$. Then defining
  $\epsilon = \epsilon' \kappa$ completes the proof.
\end{proof}

\begin{lemma}
  \label{lemma:mean_to_h_helper2}
  Let $\Theta$ be an open subset of $\mathbb R^n$.
  Let $L = \{l_1, \ldots, l_N\}$ be a set of loss functions defined over
  $\Theta$.
  Let $\theta \in \Theta$ and $\epsilon > 0$ such that $\theta + \delta \in
  \Theta$ for all $\|\delta\| \le \epsilon$.
  Let $\nu > 0$, $\beta_i(\delta) \coloneqq \frac{1}{\mu - l_i(\theta +
  \delta)}$ and $\alpha_i(\delta) \coloneqq \frac{\beta_i(\delta)}{\sum_{j=1}^N
  \beta_j(\delta)}$.
  Let $C_1$ and $C_2$ be such that $C_1 \le l_i(\theta+\delta) \le C_2$ for all $i
  \in [N]$ and $\|\delta\| \le \epsilon$.
  Then
  \begin{equation}
    \label{eq:mean_to_h:nu}
    |N \alpha_i(\delta) - 1| \le \nu
  \end{equation}
  for all $i \in [N]$, $\|\delta\| \le \epsilon$ and $\mu > \gamma$, where
  $\gamma = \max\left\{ C_2, \frac{(1+\nu) C_2 - C_1}{\nu}, \frac{C_2 - (1-\nu)
  C_1}{\nu} \right\}$.
\end{lemma}
\begin{proof}
  For Eq.~\eqref{eq:mean_to_h:nu} to hold, we must have
  \begin{subequations}
  \label{eq:mean_to_h:condition}
  \begin{gather}
    \label{eq:mean_to_h:condition1}
    N \alpha_i(\delta) - 1 \le N \max_{i,\delta} \alpha_i(\delta) - 1 \le \nu
    \\
    \label{eq:mean_to_h:condition2}
    N \alpha_i(\delta) - 1 \ge N \min_{i,\delta} \alpha_i(\delta) - 1 \ge -\nu.
  \end{gather}
  \end{subequations}

  Using the bounds $C_1$ and $C_2$, we can bound $\beta_i(\delta)$ as:
  \begin{equation}
    \label{eq:mean_to_h:beta}
    \max_{i,\delta} \beta_i(\delta) \le \frac{1}{\mu - C_2}, \quad
    \min_{i,\delta} \beta_i(\delta) \ge \frac{1}{\mu - C_1}.
  \end{equation}

  Hence, we have that Eq.~\eqref{eq:mean_to_h:condition1} can be satisfied by:
  \begin{equation}
    \label{eq:mean_to_h:alpha_upper}
    \frac{\frac{1}{\mu - C_2}}{N \frac{1}{\mu - C_1}}
    \le \frac{1+\nu}{N}
    \Rightarrow
    \frac{(1+\nu) C_2 - C_1}{\nu} \le \mu
  \end{equation}
  and Eq.~\eqref{eq:mean_to_h:condition2} can be satisfied by:
  \begin{equation}
    \label{eq:mean_to_h:alpha_lower}
    \frac{\frac{1}{\mu - C_1}}{N \frac{1}{\mu - C_2}}
    \ge \frac{1-\nu}{N}
    \Rightarrow
    \frac{C_2 - (1-\nu) C_1}{\nu} \le \mu.
  \end{equation}
  The additional value in the definition of $\gamma$ guarantees that $\mu$
  does not become invalid.
\end{proof}

\begin{proof}[Proof of Theorem~\ref{theorem:mean_to_h_main}]
  From the mean value theorem, for any $\delta$ we have that
  \begin{subequations}
  \label{eq:H_MVT}
  \begin{align}
    &H(\mu, \theta^* + \delta) - H(\mu, \theta^*)
    \\
    &= - \sum_{i=1}^N \frac{1}{\mu - l_i(\theta^* + \Delta)}
    \nabla l_i(\theta^* + \Delta) \cdot \delta
  \end{align}
  \end{subequations}
  for some $c(\delta) \in (0,1)$, where $\Delta = c(\delta) \delta$.

  Let $\epsilon_1$ be the value defined in Lemma~\ref{lemma:mean_to_h_helper}
  and define $\epsilon' = \min \{\epsilon, \epsilon_1\}$. Then restricting
  $\|\delta\| \le \epsilon'$ implies that $\|\Delta \| \le \epsilon_1$ and that
  the results in Lemma~\ref{lemma:mean_to_h_helper} hold. Therefore
  $\sum_{i=1}^N \nabla l_i(\theta^* + \Delta) \cdot \Delta \ge 0$ for all
  $\|\delta\| \le \epsilon'$.

  Then, using Lemma~\ref{lemma:mean_to_h_helper2}, the difference between the
  hypervolumes can be bounded as:
  \begin{subequations}
  \begin{align}
    &\frac{H(\mu, \theta^* + \delta) - H(\mu, \theta^*)}
    {\sum_{j=1}^N \beta_j(\Delta)}
    = - \sum_{i=1}^N \alpha_i(\Delta)
    \nabla l_i(\theta^* + \Delta) \cdot \delta
    \\
    &\le
    - \frac{1}{N} \sum_{i=1}^N (N\alpha_i(\Delta) - 1)
    \nabla l_i(\theta^* + \Delta) \cdot \delta
    \\
    &\le
    \frac{1}{N} \sum_{i=1}^N
    |N\alpha_i(\Delta) - 1|
    \|\nabla l_i(\theta^* + \Delta)\|
    \|\delta\|
    \\
    &\le
    \nu C_3 \epsilon'.
  \end{align}
  \end{subequations}
  Using the fact that $\beta_i$ is upper bounded according to
  Eq.~\eqref{eq:mean_to_h:beta}, we achieve the final bound.
\end{proof}
\subsection{Proof of Theorem~\ref{theorem:h_to_mean_main}}
\begin{lemma}
  \label{lemma:h_to_mean_helper1}
  Let $\Theta$ be an open subset of $\mathbb R^n$.
  Let $L = \{l_1, \ldots, l_N\}$ be a set of loss functions defined over
  $\Theta$.
  Let $\theta \in \Theta$ and $\epsilon > 0$ such that $\theta + \delta \in
  \Theta$ for all $\|\delta\| \le \epsilon$.
  Let $\beta_i(\delta) \coloneqq \frac{1}{\mu - l_i(\theta + \delta)}$ and
  $\alpha_i(\delta) \coloneqq \frac{\beta_i(\delta)}{\sum_{j=1}^N
  \beta_j(\delta)}$.
  Let $C_1$ and $C_2$ be such that $C_1 \le l_i(\theta+\delta) \le C_2$ for all
  $i \in [N]$ and $\|\delta\| \le \epsilon$.
  Then
  \begin{equation}
    \label{eq:h_to_mean:nu}
    |N \alpha_i(\delta) - 1| \le \nu
  \end{equation}
  for all $i \in [N]$ and $\|\delta\| \le \epsilon$, where $\nu = \max\left\{
  \frac{\mu - C_1}{\mu - C_2} - 1, 1 - \frac{\mu - C_2}{\mu - C_1} \right\}$.
\end{lemma}
\begin{proof}
  For Eq.~\eqref{eq:h_to_mean:nu} to hold, we must satisfy the conditions in
  Eq.~\eqref{eq:mean_to_h:condition}.
  Using the bounds $C_1$ and $C_2$, we can bound $\beta_i(\delta)$ as in
  Eq.~\eqref{eq:mean_to_h:beta}. Hence, we have that
  Eq.~\eqref{eq:mean_to_h:alpha_upper}
  and Eq.~\eqref{eq:mean_to_h:alpha_lower} can be satisfied by:
  \begin{equation}
    \textstyle
    \frac{\mu - C_1}{\mu - C_2} - 1 \le \nu, \quad
    1 - \frac{\mu - C_2}{\mu - C_1} \le \nu.
  \end{equation}
\end{proof}

\begin{lemma}
  \label{lemma:h_to_mean_helper2}
  Let $\Theta$ be an open subset of $\mathbb R^n$.
  Let $L = \{l_1, \ldots, l_N\}$ be a set of loss functions defined over
  $\Theta$.
  Let $\theta^* \in \Theta$ be a local maximum of $H(\mu, \theta)$ for some
  $\mu$.
  Then there is some $\epsilon > 0$ such that, for all $\xi \in (0,\epsilon]$
  and $\Delta$ with $\|\Delta\| \le \xi$, we have $\theta^* + \Delta \in \Theta$
  and
  \begin{equation}
    \sum_{i=1}^N \nabla l_i(\theta^*+\Delta) \cdot \Delta \ge
    -\nu C_3 \xi N,
  \end{equation}
  where $C_1$, $C_2$ and $C_3$ are such that $C_1 \le l_i(\theta^* + \Delta) \le
  C_2$ and $\|\nabla l_i(\theta^* + \Delta)\| \le C_3$ for all $i \in [N]$ and
  $\|\Delta\| \le \xi$ and $\nu = \max\left\{ \frac{\mu - C_1}{\mu - C_2} - 1, 1
  - \frac{\mu - C_2}{\mu - C_1} \right\}$.
\end{lemma}
\begin{proof}
  Given some $\delta$, from the mean value theorem we have that
  Eq.~\eqref{eq:H_MVT} holds for some $c(\delta) \in (0,1)$, where $\Delta =
  c(\delta) \delta$.

  Since $\Theta$ is open, there is some $\sigma > 0$ such that $\theta^* +
  \delta \in \Theta$ for all $\|\delta\| \le \sigma$.
  Since $\theta^*$ is a local maximum of $H(\mu, \theta)$, there is some
  $\epsilon' \in (0, \sigma]$ such that $H(\mu, \theta^* + \delta) \le H(\mu,
  \theta^*)$ for all $\|\delta\| \le \epsilon'$.
  Let $\kappa = \min_{\|\delta\| \le \epsilon'} c(\delta) $ and define
  $\epsilon = \epsilon' \kappa$.

  Let $\xi \in (0,\epsilon]$ and $\|\Delta\| \le \xi$.
  Using Lemma~\ref{lemma:h_to_mean_helper1}, we have that
  \begin{subequations}
  \begin{align}
    0 &\ge \frac{c(\delta)(H(\mu, \theta^* + \delta) - H(\mu, \theta^*))}
    {\sum_{j=1}^N \beta_j(\Delta)}
    \\
    &= -\sum_{i=1}^N \alpha_i(\Delta) \nabla l_i(\theta^* + \Delta) \cdot
    \Delta
    \\
    &\ge
    \left(
      \begin{aligned}
        & \textstyle
        -\frac{1}{N} \sum_{i=1}^N |(N\alpha_i(\Delta) - 1)|
        \|\nabla l_i(\theta^* + \Delta)\| \|\Delta\|
        \\
        &\quad  \textstyle
        -\frac{1}{N} \sum_{i=1}^N
        \nabla l_i(\theta^* + \Delta) \cdot \Delta
      \end{aligned}
    \right)
    \\
    &\ge
    - \nu C_3 \xi
    -\frac{1}{N} \sum_{i=1}^N
    \nabla l_i(\theta^* + \Delta) \cdot \Delta,
  \end{align}
  \end{subequations}
  which gives the bound.
\end{proof}

\begin{proof}[Proof of Theorem~\ref{theorem:h_to_mean_main}]
  Given some $\delta$, from the mean value theorem we have that
  Eq.~\eqref{eq:J_MVT} holds for some $c(\delta) \in (0,1)$, where $\Delta =
  c(\delta) \delta$.

  Let $\epsilon_1 > 0$ be the value defined in
  Lemma~\ref{lemma:h_to_mean_helper2} and define $\epsilon' = \min \{\epsilon,
  \epsilon_1\}$.
  Then restricting $\|\delta\| \le \epsilon'$ implies that $\|\Delta\| \le
  \epsilon_1$ and that the results in Lemma~\ref{lemma:h_to_mean_helper2} hold.
  Therefore, let $\xi = \epsilon'$ and we have that $\sum_{i=1}^N \nabla
  l_i(\theta^* + \Delta) \cdot \Delta \ge -\nu C_3 \epsilon' N$, which proves
  the bound.
\end{proof}

\section*{Acknowledgements}
We would like to thank CNPq and FAPESP for the financial support.

\bibliography{paper}
\bibliographystyle{icml2016}

\end{document}